\documentclass{article}

\usepackage[nonatbib,preprint]{neurips_2024}

\usepackage[utf8]{inputenc} %
\usepackage[T1]{fontenc}    %
\usepackage{hyperref}       %
\usepackage{url}            %
\usepackage{booktabs}       %
\usepackage{amsfonts}       %
\usepackage{nicefrac}       %
\usepackage{microtype}      %
\usepackage{xcolor}         %
\usepackage{amsmath}
\usepackage{longtable}%
\usepackage{booktabs}
\usepackage{multirow}
\usepackage{cleveref}
\usepackage{graphicx}
\usepackage{amsthm}
\usepackage[round]{natbib}

\newtheorem{theorem}{Theorem}
\newtheorem{corollary}{Corollary}

\title{Enhancing the Expressivity of Temporal Graph Networks through Source-Target Identification}

\author{%
  Benedict Aaron Tjandra$^{1}$\thanks{Correspondence to aaron\_tjandra@yahoo.com.}
  \phantom{1}
  Federico Barbero$^{1}$
  \phantom{1}
  Michael Bronstein$^{1}$
\\[.8em]
$^1$ Department of Computer Science, University of Oxford\\
\phantom{12}
\vspace{-.9cm}
}

\begin{document}

\maketitle

\begin{abstract}
    Despite the successful application of Temporal Graph Networks (TGNs) for tasks such as dynamic node classification and link prediction, they still perform poorly on the task of dynamic node affinity prediction -- where the goal is to predict `how much' two nodes will interact in the future. In fact, simple heuristic approaches such as persistent forecasts and moving averages over \emph{ground-truth labels} significantly and consistently outperform TGNs. Building on this observation, we find that computing heuristics \textit{over messages} is an equally competitive approach, outperforming TGN and all current temporal graph (TG) models on dynamic node affinity prediction. In this paper, we prove that no formulation of TGN can represent persistent forecasting or moving averages over messages, and propose to enhance the expressivity of TGNs by adding source-target identification to each interaction event message. We show that this modification is required to represent persistent forecasting, moving averages, and the broader class of autoregressive models over messages. Our proposed method, TGNv2, significantly outperforms TGN and all current TG models on all Temporal Graph Benchmark (TGB) dynamic node affinity prediction datasets. 
\end{abstract}

\section{Introduction}
Temporal Graph (TG) models have become increasingly popular in recent years \citep{dynamic-graphs-survey-1, dynamic-graphs-survey-2, dynamic-graphs-survey-3} due to their suitability for modeling a range of real-world systems as dynamic graphs that evolve through time, e.g. social networks, traffic networks, and physical systems \citep{social-networks-1, social-networks-2, traffic-1, traffic-2, physics-1, physics-2}. Unlike static graphs, dynamic graphs allow the addition of nodes and edges, and graph features to change over time. Despite the successes of current TG models for dynamic node classification and link prediction, they have been shown to struggle in \textit{dynamic node affinity prediction}, being significantly outperformed by simple heuristics such as persistent forecasting and moving average over ground-truth labels \citep{tgb-paper, an-empirical-evaluation-of-tgb}.

In dynamic node affinity prediction, the task is to predict a node's future `affinity' for other nodes given the temporal evolution of the graph. Informally, the affinity of a node $x$ towards a node $y$ over some time interval [$t, t+\delta$] refers to how much $x$ has interacted with $y$ over that interval. For example, if node $A$ sends 10 identical messages to node $B$ and 100 of the same messages to node $C$ over some time interval [$t, t+\delta$], then $A$ has a higher affinity for $C$ over that interval. This formulation is useful in settings such as recommender systems, e.g. predicting a user's future song preferences given their past listening history \citep{tgb-paper}. A concrete example from the Temporal Graph Benchmark (TGB) \citep{tgb-paper} is \texttt{tgbn-trade}, where nodes represent nations and edges represent the amount of goods exchanged in a single trade. In this case, the goal of dynamic node affinity prediction is to predict the amount of trade one nation would have with another nation in the next year, given the past evolution of global trading patterns. 

\paragraph{Contributions.}
This work is based on the assumption that considering past messages between two nodes is important to predict their affinity at a future time. We start by empirically validating this assumption by demonstrating that a moving average computed over a node's past messages to another node is a powerful heuristic, beating all existing TG models. Armed with this result, we ask whether Temporal Graph Networks (TGNs) \citep{tgn-paper}, a popular TG model, can represent moving averages over messages. Surprisingly, we find that no formulation of TGN can represent moving averages of any order $k$. This result implies that TGNs are unable to represent persistent forecasting (i.e. the simple heuristic of outputting the most recent message between a pair of nodes), indicating a substantial weakness in its design. To remedy this, we propose to modify TGN by adding source-target identification to each interaction event message. We prove that our method, TGNv2, is strictly more expressive than TGN as it is able to represent persistent forecasting, moving averages, and autoregressive models. Further, we show that TGNv2 significantly outperforms all current TG models on all TGB datasets on dynamic node affinity prediction. 

\section{The Hidden Limitation of Temporal Graph Networks}
This work is motivated by our observation that computing moving averages over past messages, despite still lagging behind moving average over ground-truth labels, is a competitive heuristic that outperforms all current TG models on every node affinity prediction dataset (\Cref{main-results}). Given an order $k \in \mathbb{N}^+$, the moving average heuristic over past messages for node affinity prediction is defined as: 
\begin{equation*}\label{moving-average-raw}
    \mathbf{\hat{y}}_t[u, v] = \frac{1}{k} \sum_{t' \in M(u, v, t, k)} e_{uv}(t')
\end{equation*}
where $M(u, v, t, k)$ returns $k$ ordered timestamps that constitute the $k$ most-recent messages sent from node $u$ to node $v$ up to time $t$ and $e_{uv}(t')$ is the scalar event message passed from node $u$ to node $v$ during their interaction at time $t'$. Given this observation, we focus on TGNs and study if there exists a formulation of TGN that can \emph{exactly represent a moving average of order $k$}. Our first important result is proving that this cannot be the case: 
\begin{theorem}\label{thm:no-moving-average-tgn}
    No formulation of TGN can represent a moving average of order $k \in \mathbb{N}^+$ for any temporal graph with a bounded number of vertices.
\end{theorem}
We prove \Cref{thm:no-moving-average-tgn} in \Cref{appendix:thm:no-moving-average-tgn}. In short, the proof constructs a minimal example of two nodes in two graphs sending different messages. We show that TGNs cannot distinguish the two nodes, leading them to compute the same moving average for both nodes.  This is a direct consequence of the permutation-invariance of TGNs, which renders them unable to discriminate between senders and receivers of messages, and in turn incapable of capturing important functions.

Since the above theorem holds for any $k$ and persistent forecasting is equivalent to a moving average when $k = 1$, it follows that TGNs cannot represent persistent forecasting. Proceeding similarly to our proof for \Cref{thm:no-moving-average-tgn}, we can show that, more generally, TGNs cannot represent the class of autoregressive functions (proof in \Cref{appendix:tgn-autoregressive-proof}):
\begin{corollary}
    No formulation of TGN can represent an autoregressive model of order $k \in \mathbb{N}^+$ for any temporal graph with a bounded number of vertices.
\end{corollary}

\subsection{TGNv2: Increasing the expressive power of TGNs}
The main problem with TGN lies in the construction of the messages when an event occurs. In TGN, for every interaction between nodes $i$ and $j$, two messages are constructed: 
\begin{align*}
    \mathbf{m}_i(t) = \text{msg}_s(\mathbf{s}_i(t^-), \mathbf{s}_j(t^-), \phi(\Delta t), e_{ij}(t)); 
    && \mathbf{m}_j(t) = \text{msg}_d(\mathbf{s}_j(t^-), \mathbf{s}_i(t^-), \phi(\Delta t), e_{ij}(t))
\end{align*}
If we look closely, however, we can see that each message does not contain the source or the destination of the message. Not only does this make it impossible for the memory vectors to have an imprint of past interactions, but this also renders TGNs to be \textit{invariant} to the identities of the senders and receivers of messages--a property that is undesirable for dynamic node affinity prediction. To address this issue, we introduce TGNv2, where we modify the message construction of TGNs to include source-target identification:
\begin{align*}
    \mathbf{m}_i(t) &= \text{msg}_s(\mathbf{s}_i(t^-), \mathbf{s}_j(t^-), \phi_t(\Delta t), e_{ij}(t), \phi_n(i), \phi_n(j)) \\
    \mathbf{m}_j(t) &= \text{msg}_d(\mathbf{s}_j(t^-), \mathbf{s}_i(t^-), \phi_t(\Delta t), e_{ij}(t), \phi_n(j), \phi_n(i))
\end{align*}
Here, we map all nodes to an arbitrary, but fixed node index, and $\phi_n \in \mathbb{R} \to \mathbb{R}^d$ is an encoder function for node indices, similar to $\phi_t$. Incoming nodes that have not been encountered before are assigned to fresh, unused node indices as the graph evolves.  This modification is a way to break the permutation-invariance of TGN, which is necessary to compute moving averages and autoregressive models. We are now able to prove:
\begin{theorem}\label{thm:tgnv2-proof}
     There exists a formulation of TGNv2 that can represent persistent forecasting, moving average of order $k \in \mathbb{N}^+$, or any autoregressive model of order $k \in \mathbb{N}^+$ for any temporal graph with a bounded number of vertices.
\end{theorem}
Our proof of \Cref{thm:tgnv2-proof}  (\Cref{appendix:tgnv2-proof}) leverages the existence of the node identification to `index' into the memory vector to store information. From this, it follows that TGNv2 is strictly more expressive than TGN, as TGN is a special case of TGNv2.
\section{Experiments}
\Cref{main-results} shows our experimental results on the TGB benchmark. The top 3 rows are simple heuristics over ground-truth labels / ground-truth messages. `Persistent Frcst (L)' and `Moving Average (L)' refer to persistent forecasting and moving average over ground-truth labels respectively; while `Moving Avg (M)' is a moving average over messages. The rest of the rows correspond to TG models. We describe the experimental details in \Cref{appendix:experiment-details}.

Evidently from \Cref{main-results}, Moving Average (M) is a competitive method that outperforms all TG models. TGN (tuned) denotes the TGN that we trained using the same set of hyperparameters for TGNv2. Though TGN enjoys a performance boost with this set of hyperparameters, TGNv2 significantly outperforms TGN and all TG models on all datasets. Further, we can see that TGNv2 performs comparably to Moving Avg (M) on \texttt{tgbn-trade, tgbn-genre, tgbn-reddit} while TGN is beaten by Moving Avg (M) on all datasets. 
\begin{table}[!h]
  \caption{Main results. $\dagger$ are results obtained from \cite{tgb-paper}, while $\ddagger$ are obtained from \cite{an-empirical-evaluation-of-tgb}. TGNv2 outperforms all current TG models by a large margin.}
  \label{main-results}.  
  \centering
  \resizebox{\textwidth}{!}{
      \begin{tabular}{l| ll  ll  ll  ll }
        \toprule
        \multirow{3}{*}{Method} & \multicolumn{2}{c}{\texttt{tgbn-trade}} & \multicolumn{2}{c}{\texttt{tgbn-genre}} & \multicolumn{2}{c}{\texttt{tgbn-reddit}} & \multicolumn{2}{c}{\texttt{tgbn-token}}\\
        & \multicolumn{2}{c}{\textbf{NDCG @ 10 $\uparrow$}} & \multicolumn{2}{c}{\textbf{NDCG @ 10 $\uparrow$}} & \multicolumn{2}{c}{\textbf{NDCG @ 10 $\uparrow$}} & \multicolumn{2}{c}{\textbf{NDCG @ 10 $\uparrow$}} \\
        & Validation & Test & Validation & Test & Validation & Test & Validation & Test \\
        \midrule
        $\text{Persistent Frcst (L)}^\dagger$ & \textbf{0.860} &  \textbf{0.855}  & 0.350 & 0.357 & 0.380 & 0.369 & 0.403 & 0.430  \\
        $\text{Moving Avg (L)}^\dagger$ & 0.841 & 0.823 & \textbf{0.499}  &  \textbf{0.509} & \textbf{0.574} & \textbf{0.559} & \textbf{0.491} & \textbf{0.508} \\
        $\text{Moving Avg (M)}$ & 0.793 & 0.777 & 0.478  &  0.472 & 0.499 & 0.481 & 0.402 & 0.415 \\
        \midrule
        $\text{JODIE}^\ddagger$ &  $0.394_{\pm 0.05}$  &  $0.374_{\pm 0.09}$  &  $0.358_{\pm 0.03}$  & $0.350_{\pm 0.04}$  & 	0.345\textsubscript{$\pm 0.02$} & 0.314\textsubscript{$\pm0.01$} & & \\ 
        $\text{TGAT}^\ddagger$ &  $0.395_{\pm 0.14}$  &  $0.375_{\pm 0.07}$  &  $0.360_{\pm 0.04}$  & $0.352_{\pm 0.03}$  & 	0.345\textsubscript{$\pm 0.01$} & 0.314\textsubscript{$\pm0.01$} & & \\ 
        $\text{CAWN}^\ddagger$ &  $0.393_{\pm 0.07}$  &  $0.374_{\pm 0.09}$  & &  & &  & & \\ 
        $\text{TCL}^\ddagger$ &  $0.394_{\pm 0.11}$  &  $0.375_{\pm 0.09}$  &  $0.362_{\pm 0.04}$  & $0.354_{\pm 0.02}$  & 	0.347\textsubscript{$\pm 0.01$} & 0.314\textsubscript{$\pm0.01$} & & \\ 
        $\text{GraphMixer}^\ddagger$ &  $0.394_{\pm 0.17}$  &  $0.375_{\pm 0.11}$  &  $0.361_{\pm 0.04}$  & $0.352_{\pm 0.03}$  & 	0.347\textsubscript{$\pm 0.01$} & 0.314\textsubscript{$\pm0.01$} & & \\ 
        $\text{DyGFormer}^\ddagger$ &  $0.408_{\pm 0.58}$  &  $0.388_{\pm 0.64}$  &  $0.371_{\pm 0.06}$  & $0.365_{\pm 0.20}$  & 	0.348\textsubscript{$\pm 0.02$} & 0.316\textsubscript{$\pm0.01$} & & \\ 
        $\text{DyRep}^\dagger$ &  $0.394_{\pm 0.001}$  &  $0.374_{\pm 0.001}$  &  $0.357_{\pm 0.001}$  & $0.351_{\pm 0.001}$  & 	0.344\textsubscript{$\pm 0.001$} & 0.312\textsubscript{$\pm0.001$} & 0.151\textsubscript{$\pm0.006$} & 0.141\textsubscript{$\pm 0.006$} \\ 
        $\text{TGN}^\dagger$ & $0.395_{\pm 0.002}$  &  $0.374_{\pm 0.001}$  &  $0.403_{\pm 0.010}$ & $0.367_{\pm 0.058}$ & 0.379\textsubscript{$\pm$ 0.004} & 0.315\textsubscript{$\pm$ 0.020} & 0.189\textsubscript{$\pm 0.005$} & 0.169\textsubscript{$\pm 0.006$}\\
        TGN (tuned) & 0.445\textsubscript{$\pm 0.009$}&  0.409\textsubscript{$\pm 0.005$}&  0.443\textsubscript{$\pm 0.002$} & 0.423\textsubscript{$\pm 0.007$} & 0.482\textsubscript{$\pm 0.007$} & 0.408\textsubscript{$\pm 0.006$} & 0.251\textsubscript{$\pm 0.000$} & 0.200\textsubscript{$\pm 0.005$} \\
        TGNv2 (ours) & \textbf{0.807\textsubscript{$\pm$ 0.006}} & \textbf{0.735\textsubscript{$\pm$ 0.006}} & \textbf{0.481\textsubscript{$\pm$ 0.001}} & \textbf{0.469\textsubscript{$\pm$ 0.002}} & \textbf{0.544\textsubscript{$\pm$ 0.000}} & \textbf{0.507\textsubscript{$\pm 0.002$}} & \textbf{0.321\textsubscript{$\pm$ 0.001}} & \textbf{0.294\textsubscript{$\pm$ 0.001}} \\
        \bottomrule
      \end{tabular}
    }
\end{table}

\section{Related Work}
We believe this work is the first to address the limitations of TG models in dynamic node affinity prediction. \citet{tgb-paper} were the first to point out this problem, highlighting that TGN and DyRep \citep{dyrep-paper} are outperformed by heuristics over ground-truth labels. \cite{an-empirical-evaluation-of-tgb} extended this work and found that a suite of other TG models (JODIE \citep{jodie}, TGAT \citep{tgat-paper}, CAWN \citep{cawn}, TCL \citep{tcl}, GraphMixer \citep{graph-mixer}, and DyGFormer \citep{le-yu-2023-neurips}) all underperform in dynamic node affinity prediction. Despite still lagging behind heuristics over ground-truth labels, TGNv2 significantly outperforms all of the methods above, constituting what we believe to be \textbf{the first positive result in improving TG models for dynamic node affinity prediction}. Our method of augmenting TGNs with source-target identification to increase expressivity is most similar to the work of \citet{sato-port-numbering}, where they increased the expressivity of static GNNs via port numbering. Relatedly, other works demonstrated that breaking the permutation-invariance of static GNNs (e.g. by using RNNs to aggregate messages) led to empirical benefits \citep{rnar, graphsage}.

\section{Conclusion}
In this paper, we proposed to augment TGN with source-target identification. We proved that TGNv2 is strictly more expressive than TGN and consequently showed that TGNv2 achieves significantly higher performance than current TG models across all dynamic node affinity prediction datasets from TGB. In the future, we would like to close the remaining empirical gap between TGNv2 and the heuristics approaches. We believe this is because we formulated our message aggregator to output the last message, which was necessary to compare our results fairly with prior TGN experiments (\Cref{appendix:experiment-details}). To address this, we hope to explore more expressive aggregation functions. Moreover, we would like to further develop our work by studying TGNv2 on other TG tasks, such as dynamic link prediction.

\bibliography{refs}

\newpage
\appendix

\section{TGN Recap}
For the reader's convenience, we restate the core modules of TGN. For a more thorough explanation of each module, we refer the reader to the original paper \citep{tgn-paper}.

\paragraph{Message Function.}
For each interaction between $i$ and $j$, we construct two messages $\text{msg}_s$ and $\text{msg}_d$: 
\begin{align*}
\label{message-function-update}
    \mathbf{m}_i(t) = \text{msg}_s(\mathbf{s}_i(t^-), \mathbf{s}_j(t^-), \phi(\Delta t), e_{ij}(t)); 
    && \mathbf{m}_j(t) = \text{msg}_d(\mathbf{s}_j(t^-), \mathbf{s}_i(t^-), \phi(\Delta t), e_{ij}(t))
\end{align*}
We can also opt to construct node messages if node events exist:
\begin{align*}
    \mathbf{m}_i(t) = \text{msg}_n(\mathbf{s}_i(t^-), t, \mathbf{v}_i(t))
\end{align*}

\paragraph{Message Aggregator.}
\begin{align*}
    \bar{\mathbf{m}}_i(t) = \text{agg}(\mathbf{m}_i(t_1), ..., \mathbf{m}_i(t_b))
\end{align*}

\paragraph{Memory Updater.}
\begin{align*}
    \mathbf{s}_i(t) = \text{mem}(\bar{\mathbf{m}}_i(t), \mathbf{s}_i(t^-)))
\end{align*}

\paragraph{Embedding}
\begin{align*}
    \mathbf{z}_i(t) = g\bigl( \bigl \{ \bigl \{ h(\mathbf{s}_i(t), \mathbf{s}_j(t), e_{ij}, \mathbf{v}_i(t), \mathbf{v}_j(t)) : j \in \mathcal{N}_i^L([0, t])  \bigr\} \bigr \} \bigr)
\end{align*}
Here, $h$ is a learnable function, $g$ is a permutation-invariant function such as a sum or mean, and $L$ corresponds to the number of layers used for temporal message passing. We note that our formulation of the embedding layer is a more general version than in the TGN paper, and we can recover the original formulation by setting $g$ to be a sum. 

\section{Proofs}
\subsection{Proof for Theorem 1}\label{appendix:thm:no-moving-average-tgn}
\setcounter{theorem}{0}
\begin{theorem}
    No formulation of TGN can represent a moving average of order $k \in \mathbb{N}^+$ for any temporal graph with a bounded number of vertices.
\end{theorem}
\begin{proof}
The main idea of the proof is that TGNs are unable to distinguish nodes whose messages are identical in every form but have different senders and/or recipients. To show this, we construct a temporal graph $G$ with 3 nodes (Node 1, 2, and 3) and `flip' it in such a way to yield a $G'$ such that Node 1 in $G'$ is sending different messages when compared to Node 1 in $G$ but is indistinguishable from Node 1 in $G$ from the point of view of TGNs (\Cref{fig:tgn-indistinguishable}).

\begin{figure}[htbp]
\centering
  \label{fig:tgn-indistinguishable}
  \caption{Graphs $G$ and $G'$. Clearly Node 1 in $G$ and $G'$ are sending different sequences of messages, but TGNs are unable to distinguish them.}
  \includegraphics[width=\textwidth, height=0.23\textheight, keepaspectratio]{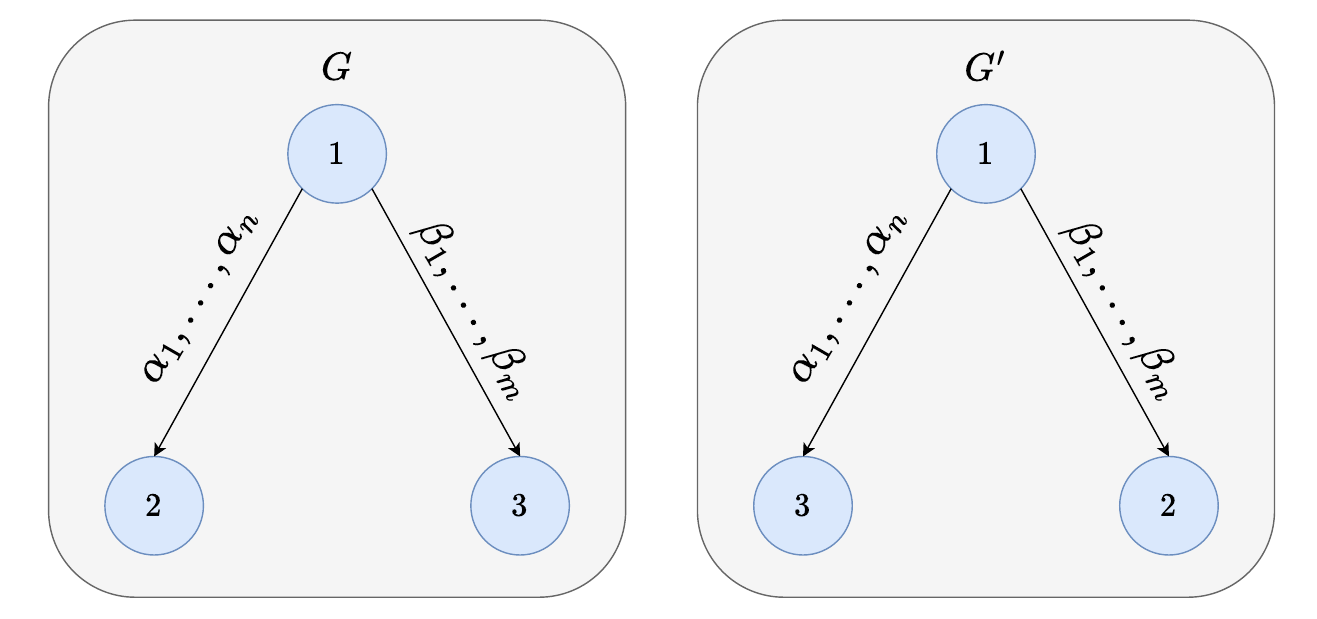}
\end{figure}

We proceed by way of contradiction. Assume that there exists a particular formulation of TGN that can implement a moving average of order $k$ for any temporal graph with a bounded number of vertices. We initialise all memory vectors to be the zero vector, as per TGN's original formulation. Now, consider the following sequence of events that implicitly define the temporal graph $G$: 
\begin{enumerate}
    \item Node 1 sends Node 2 a series of $n$ events with features $\alpha_1, \dots, \alpha_n$ at timestamps $t_1, \dots, t_n$.
    \item Node 1 sends Node 3 a series of $m$ events with features $\beta_1, \dots, \beta_m$ at timestamps $t_{n + 1}, \dots, t_{n + m}$.
\end{enumerate}
where $t_1 < \dots < t_n < t_{n + 1} < \dots < t_{n + m}$, $n \ge k$, and $m \ge k$. Let $\bar{\alpha} = \frac{\alpha_{n - k + 1} + \dots + \alpha_n}{k}$, the moving average of order $k$ of $\alpha_1, ..., \alpha_n$. Similarly, define $\bar{\beta}$ to be the moving average of order $k$ of $\beta_1, ..., \beta_m$. Suppose we compute Node 1's embedding at time $t_T$ where $t_T > t_{n + m}$. We assume that no node updates are done, and all $\mathbf{v}_1(t) = \mathbf{v}_2(t) = \mathbf{v}_3(t)$ for all $t$. Node 1 receives $n$ messages from its interactions with Node 2: 
\begin{equation*}
   \mathbf{m}_1(t_i) = \text{msg}_s(\mathbf{0}, \mathbf{0}, t_i , \alpha_i) \qquad \forall i \in [ 1, \dots, n ]
\end{equation*}
Similarly, Node 1 receives $m$ messages from its interactions with Node 3:
\begin{equation*}
   \mathbf{m}_1(t_{n + i}) = \text{msg}_s(\mathbf{0}, \mathbf{0}, t_{n + i} , \beta_i)  \qquad \forall i \in [1, \dots, m] 
\end{equation*}
Node 1 then aggregates the messages it receives and updates its memory: 
\begin{align*}
    \mathbf{\tilde{m}}_1(t_T) &= \text{agg}(\mathbf{m}_1(t_1), \dots, \mathbf{m}_1(t_{n + m})) \\
    \mathbf{s}_1(t_T) &= \text{mem}(\mathbf{\tilde{m}}_1(t_T), \mathbf{0})
\end{align*}
We can do the same set of calculations for Node 2 and Node 3: 
\begin{align*}
    \mathbf{m}_2(t_{i}) &= \text{msg}_d(\mathbf{0}, \mathbf{0}, t_i, \alpha_i)  \qquad \forall i \in [ 1, \dots, n ] \\
    \mathbf{m}_3(t_{n + i}) &= \text{msg}_d(\mathbf{0}, \mathbf{0}, t_{n + i}, \beta_i)  \qquad \forall i \in [ 1, \dots, m ] \\
    \mathbf{\tilde{m}}_2(t_T) &= \text{agg}(\mathbf{m}_2(t_1), \dots, \mathbf{m}_2(t_{n})) \\ 
    \mathbf{\tilde{m}}_3(t_T) &= \text{agg}(\mathbf{m}_3(t_{n + 1}), \dots, \mathbf{m}_3(t_{n + m})) \\
    \mathbf{s}_2(t_T) &= \text{mem}(\mathbf{\tilde{m}}_2(t_T), \mathbf{0}) \\
    \mathbf{s}_3(t_T) &= \text{mem}(\mathbf{\tilde{m}}_3(t_T), \mathbf{0}) 
\end{align*}
Then, as we have assumed no node events have occurred, and all $\mathbf{v}_i(t_T)$ are the same, then we can ignore them during the embedding computation: 
\begin{align*}
    \mathbf{z}_1(t_T) &= g\bigl( \bigl\{ \bigl\{ h(\mathbf{s}_1(t_T), \mathbf{s}_2(t_T), \alpha_i) : i \in [1, \dots n] \bigr\} \bigr\} \cup \bigl\{ \bigl\{ h(\mathbf{s}_1(t_T), \mathbf{s}_3(t_T), \beta_i) : i \in [1, \dots m] \bigr\} \bigr\} \bigr) \\
                      &= [0, \bar{\alpha}, \bar{\beta}]^T
\end{align*}
as per our assumption. Consider now the flipped temporal graph $G'$ with events: 
\begin{enumerate}
    \item Node 1 sends Node 3 a series of $n$ events with features $\alpha_1, \dots, \alpha_n$ at timestamps $t_1, \dots, t_n$.
    \item Node 1 sends Node 2 a series of $m$ events with features $\beta_1, \dots, \beta_m$ at timestamps $t_{n + 1}, \dots, t_{n + m}$.
\end{enumerate}
where $t_1 < \dots < t_n < t_{n + 1} < \dots < t_{n + m}$, $n \ge k$, $m \ge k$ as before. Suppose we are again to compute the node embeddings at time $t_T > t_{n + m}$. Node 1 receives $n$ messages from its interactions with Node 3:
\begin{align*}
   \mathbf{m}'_1(t_i) = \text{msg}_s(\mathbf{0}, \mathbf{0}, t_i , \alpha_i) \qquad \forall i \in [ 1, \dots, n ]
\end{align*}
Node 1 receives $m$ messages from its interactions with Node 2: 
\begin{align*}
   \mathbf{m}'_1(t_i) = \text{msg}_s(\mathbf{0}, \mathbf{0}, t_{n + i} , \beta_i) \qquad \forall i \in [ 1, \dots, m ]
\end{align*}
But observe that Node 1 receives the same set of messages as it did in $G$. Therefore, $\mathbf{s}'_1(t_T)$ must be equal to $\mathbf{s}_1(t_T)$:
\begin{align*}
    \mathbf{\tilde{m}}'_1(t_T) &= \text{agg}(\mathbf{m}'_1(t_1), \dots, \mathbf{m}'_1(t_{n + m})) \\
                               &= \text{agg}(\mathbf{m}_1(t_1), \dots, \mathbf{m}_1(t_{n + m})) \\
                               &= \mathbf{\tilde{m}_1}(t_T) \\
    \mathbf{s}'_1(t_T) &= \text{mem}(\mathbf{\tilde{m}}'_1(t_T), \mathbf{0}) \\
                       &= \text{mem}(\mathbf{\tilde{m}}_1(t_T), \mathbf{0}) \\
                       &= \mathbf{s}_1(t_T)
\end{align*}
Further, we can see that Node 3 in $G'$ receives the same set of messages as Node 2 in $G$, and Node 2 in $G'$ receives the same set of messages as Node 3 in $G$. Following the same reasoning, we can conclude that: 
\begin{align*}
    \mathbf{\tilde{m}}'_2(t_T) &= \text{agg}(\mathbf{m}'_2(t_{n + 1}), \dots, \mathbf{m}'_2(t_{n + m})) \\
                               &= \text{agg}(\mathbf{m}_3(t_{n + 1}), \dots, \mathbf{m}_3(t_{n + m})) \\
                               &= \mathbf{\tilde{m}_3}(t_T) \\
    \mathbf{s}'_2(t_T) &= \text{mem}(\mathbf{\tilde{m}}'_2(t_T), \mathbf{0}) \\
                       &= \text{mem}(\mathbf{\tilde{m}}_3(t_T), \mathbf{0}) \\
                       &= \mathbf{s}_3(t_T) \\
    \mathbf{\tilde{m}}'_3(t_T) &= \text{agg}(\mathbf{m}'_3(t_{1}), \dots, \mathbf{m}'_3(t_{n})) \\
                               &= \text{agg}(\mathbf{m}_2(t_1), \dots, \mathbf{m}_2(t_{n})) \\
                               &= \mathbf{\tilde{m}_2}(t_T) \\
    \mathbf{s}'_3(t_T) &= \text{mem}(\mathbf{\tilde{m}}'_3(t_T), \mathbf{0}) \\
                       &= \text{mem}(\mathbf{\tilde{m}}_2(t_T), \mathbf{0}) \\
                       &= \mathbf{s}_2(t_T)                       
\end{align*}
Therefore:
\begin{align*}
    \mathbf{z}'_1(t_T) &= g \bigl( \bigl\{ \bigl\{ h(\mathbf{s}'_1(t_T), \mathbf{s}'_3(t_T), \alpha_i) : i \in [1, \dots n] \bigr\} \bigr\} \cup \bigl\{ \bigl\{ h(\mathbf{s}'_1(t_T), \mathbf{s}'_2(t_T), \beta_i) : i \in [1, \dots m] \bigr\} \bigr\} \bigr) \\
                       &= g\bigl( \bigl\{ \bigl\{ h(\mathbf{s}_1(t_T), \mathbf{s}_2(t_T), \alpha_i) : i \in [1, \dots n] \bigr\} \bigr\} \cup \bigl\{ \bigl\{ h(\mathbf{s}_1(t_T), \mathbf{s}_3(t_T), \beta_i) : i \in [1, \dots m] \bigr\} \bigr\} \bigr) \\
                       &= \mathbf{z}_1(t_T) \\
                      &= [0, \bar{\alpha}, \bar{\beta}]^T
\end{align*}
which is a contradiction, as a moving average of order $k$ would've computed $[0, \bar{\beta}, \bar{\alpha}]^T$ for $G'$.
\end{proof}

\subsection{Proof for Corollary 2}\label{appendix:tgn-autoregressive-proof}
\begin{corollary}
    No formulation of TGN can represent an autoregressive model of order $k \in \mathbb{N}^+$ for any temporal graph with a bounded number of vertices.
\end{corollary}
\begin{proof}
Our proof for Theorem 2 proceeds very similarly to Theorem 1. Notice that in our proof of Theorem 1, we did not make use of the fact that $\bar{\alpha}$ and $\bar{\beta}$ are moving averages. Therefore, if our auto-regressive model has weights $w_1, \dots, w_{k}$, then we can define $\bar{\alpha}$ and $\bar{\beta}$ to be:
\begin{align*}
    \bar{\alpha} &= \sum_{i = 1}^{k} w_i \alpha_{n - i} \\ 
    \bar{\beta} &= \sum_{i = 1}^{k} w_i \beta_{n - i}
\end{align*}
and consequently proceeding in the same manner as we did in Theorem 1.
\end{proof}

\subsection{Proof for Theorem 3}\label{appendix:tgnv2-proof}
\begin{theorem}
     There exists a formulation of TGNv2 that can represent persistent forecasting, moving average of order $k \in \mathbb{N}^+$, or any autoregressive model of order $k \in \mathbb{N}^+$ for any temporal graph with a bounded number of vertices.
\end{theorem}
\begin{proof}
We first prove the theorem for the case of moving averages of order $k$, and extend that to apply to persistent forecasting and autoregressive models. Let the maximum number of nodes encountered in the temporal graph be $n$, and assign each node an identifier such that each node is uniquely identified by an $i \in [0, \dots, n - 1]$. We initialise all memory vectors $\mathbf{s}_i$ to be $\mathbf{0} \in \mathbb{R}^{nk}$. 

Next, denote $e_{ij}^{(l)}$ be the feature of the $l$-th message that $i$ sends to $j$, and let $M(i, j, t)$ return the index of the most recent message that $i$ sent to $j$ at time $t$. We assume that the batch size is 1, which means that as soon as a message is sent, we update the memory vectors. Further, for the time being, we ignore the formulation of $\text{msg}_d$, and assume that all the messages received by the aggregator are messages constructed by $\text{msg}_s$. \\

Our goal is to find a formulation of TGNv2 such that, for each timestamp $t$, it computes $\mathbf{z}_i(t)[j] = \frac{1}{k}\sum_{x = 0}^{k - 1} e_{ij}^{(M(i, j, t) - x)}$. We assume that the moving average is defined for all values of $t$ by letting $e_{ij}^{(l)} = 0$ for all negative $l$. We now concretely define the formulation of TGNv2, and subsequently show that it computes the moving average. Now, suppose some node $i$ sends $j$ a message at time $t$. Let $\text{msg}_s$ be formulated as: 
\[
    \text{msg}_s(\mathbf{s}_i(t^-), \mathbf{s}_j(t^-), \phi_t(\Delta t), e_{ij}(t), \phi_n(i), \phi_n(j))= [e_{ij}(t), j]
\]
i.e. $\text{msg}_s$ simply outputs a 2-element vector with the feature of the event message and the index of the destination node. Since our batch size is 1, the aggregator  only receives at most one message. We let the aggregator be the identity function. The main idea of the proof is in the formulation of the memory module, which takes advantage of the node index of the destination to `store' the newest feature message between $i$ and $j$. 

We introduce some machinery to aid our formulation. Consider a block matrix $\mathbf{B} \in \mathbb{R}^{nk \times nk}$ with $n$ matrices $\mathbf{B}_1, \dots, \mathbf{B}_n \in \mathbb{R}^{k \times k}$ in its diagonal: 
\[
    \mathbf{B} = \begin{bmatrix} 
    \mathbf{B}_1 & \dots  & \mathbf{0} \\
    \vdots & \ddots & \vdots\\
    \mathbf{0} & \dots  & \mathbf{B}_n 
    \end{bmatrix}
\]
Define the block-permutation matrix $\mathbf{P} \in \mathbb{R}^{nk \times nk}$:
\[
    \mathbf{P} = \begin{bmatrix} 
    \mathbf{0} & \mathbf{0} &\dots  & \mathbf{I} \\
    \mathbf{I} & \mathbf{0} & \dots  & \mathbf{0} \\
    \mathbf{0} & \mathbf{I} & \dots  & \mathbf{0} \\
    \vdots & \vdots & \ddots & \vdots\\
    \mathbf{0} & \dots  & \mathbf{I} & \mathbf{0}
    \end{bmatrix}
\]
where $\mathbf{I}$ is the $\mathbf{R}^{k \times k}$ identity matrix. Observe that $\mathbf{P}\mathbf{B}\mathbf{P}^T$ cyclically shifts the order of $\mathbf{B}_1, \dots, \mathbf{B}_n$ in $\mathbf{B}$ by one: 
\begin{align*}
    \mathbf{P}\mathbf{B}\mathbf{P}^T = \begin{bmatrix} 
    \mathbf{B}_n & \mathbf{0} & \dots  & \mathbf{0} \\
    \mathbf{0} & \mathbf{B}_1 & \dots & \mathbf{0} \\
    \vdots & \vdots & \ddots & \vdots\\
    \mathbf{0} & \dots  & \mathbf{0} &  \mathbf{B}_{n - 1} 
    \end{bmatrix}
\end{align*}
Similarly, define the permutation matrix $\mathbf{Q}$ that analogously shifts the elements of a vector cyclically by 1: 
\[
    \mathbf{Q} = \begin{bmatrix} 
    0 & 0 &\dots  & 1 \\
    1 & 0 & \dots  & 0 \\
    0 & 1 & \dots  & 0 \\
    \vdots & \vdots & \ddots & \vdots\\
    0 & \dots  & 1 & 0
    \end{bmatrix}
\]

In order to compute the moving average, then each time we receive a message, we need to `make room' in our memory vector to store the message. We introduce the shift matrix $\mathbf{S}$, which is a  $k \times k$ matrix that is obtained by taking the $(k - 1) \times (k - 1)$ identity matrix and sufficiently padding the top row and rightmost column with zeroes which, when applied to a vector $\mathbf{v} \in \mathbb{R}^k$, keeps the top $k-1$ elements and discards the last element:
\[
    \mathbf{S} = \begin{bmatrix} 
    0 & 0 & \dots  & 0 & 0\\
    1 & 0 & \dots & 0 & 0  \\
    \vdots & \ddots & \ddots &\vdots & \vdots\\
    0 & 0 & \dots  & 1 & 0 
    \end{bmatrix}
\]
Define the generator block matrix $\mathbf{X} \in \mathbb{R} ^ {nk \times nk}$ that has the same form as $\mathbf{B}$ and consists of $n$ matrices $\mathbf{B}_1, \dots, \mathbf{B}_n$, but with $\mathbf{B}_1 = \mathbf{S}$ and all other $\mathbf{B}_i = \mathbf{I}$ : 
\[
    \mathbf{X} = \begin{bmatrix} 
    \mathbf{S} & \mathbf{0} & \dots  & \mathbf{0} \\
    \mathbf{0} &  \mathbf{I} &  \dots  & \mathbf{0} \\
    \vdots & \vdots & \ddots & \vdots\\
    \mathbf{0} & \dots  & \mathbf{0} & \mathbf{I} 
    \end{bmatrix}
\]
Similarly, define the generator vector $\mathbf{y} \in \mathbb{R}^{nk} = [1, 0, \dots, 0]^T$. Next, define $f(j) = (\mathbf{P})^j\mathbf{X}(\mathbf{P}^T)^j$, which cyclically shifts the block matrices in $\mathbf{X}$ a total number of $j$ times, and $p(j) = \mathbf{Q}^j \mathbf{y}$, which cyclically shifts the elements of $\mathbf{y}$ a total number of $j$ times. Now, let the memory module be: 
\begin{align*}
    \mathbf{s}_i(t) &= \text{mem}(\mathbf{\bar{m}}_i(t), \mathbf{s}_i(t^{-})) \\
                    &= \text{mem}([e_{ij}(t), j], \mathbf{s}_i(t^{-})) \\ 
                    &= f(j) \cdot  \mathbf{s}_i(t^{-})) + p(j) \cdot e_{ij}(t) \\
\end{align*}
It is quite easy to see, via a straightforward induction, that at every timestamp $t$, $\mathbf{s}_i(t)$ stores the $k$ most recent messages sent to $j$ in the `subarray' $\mathbf{s}_i(t)[jk: jk  + k - 1]$. Consequently, making $\mathbf{z}_i(t)$ compute the moving average is straightforward. Define the aggregator matrix $\mathbf{A} \in \mathbb{R}^{n \times nk}$ to be: 
\[
    \mathbf{A}[m, n] = \begin{cases}
    \frac{1}{k} &\text{if $mk \le  n \le mk + k - 1$}\\
    0 &\text{otherwise}
    \end{cases}
\]
and let our embedding layer be defined as:
\begin{align*}
    \mathbf{z}_i(t) &= g(\{\{ h(\mathbf{s}_i(t), \mathbf{s}_j(t), e_{ij}, \mathbf{v}_i(t), \mathbf{v}_j(t)) : j \in \mathcal{N}_i^L([0, t]  \}\}) \\ 
                    &= \frac{1}{|\mathcal{N}_i^L([0, t]|} \sum_{j = 0} ^ {|\mathcal{N}_i^L([0, t]|} \mathbf{A} \cdot \mathbf{s}_i(t) \\
                    &= \mathbf{A} \cdot \mathbf{s}_i(t)
\end{align*}
which is the moving average of order $k$, as multiplying $\mathbf{A}$ with $\mathbf{s}_i(t)$ has the effect of summing the $k$ most recent messages for each node and multiplying the sum by $\frac{1}{k}$. From this, we can see that the theorem holds for persistent forecasting as persistent forecasting is moving average with $k = 1$. Subsequently, we can adapt our proof above to hold for autoregressive models of any order $k$ by formulating the aggregator matrix $\mathbf{A}$ to have the autoregressive weights $w_k, \dots, w_{1}$ in entries where 1 is present.

In our constructions above, we assumed that our batch size is 1 and we ignored the fact that nodes are receiving messages from $\text{msg}_d$. To adapt the proof for an arbitrary batch size, we can define the aggregator module to concatenate all incoming messages, and then during the memory update, we `unpack' this concatenation and apply our logic above for each message. Finally, to handle messages from $\text{msg}_d$, we can expand the size of the message vector by 1 to include a `tag' that is nonzero if and only if the message originates from $\text{msg}_d$. Then, the aggregator module can drop messages from $\text{msg}_d$ by inspecting this tag, leaving us with messages from $\text{msg}_s$ -- which we have shown how to handle.

\end{proof}

\section{Experiment Details}\label{appendix:experiment-details}
The code to reproduce our experiments can be found at \url{https://github.com/batjandra/TGNv2-NeurReps}.\\

For both TGN and TGNv2, we utilise the same choices of core modules where applicable and use the same hyperparameters in order to make the results as comparable as possible. We repeat each experiment run three times with three different random seeds, each time picking the best-performing model on the validation set, and reporting the mean and standard deviation NDCG @ 10 on both the validation and test set. For our experiments, our choice of core module largely follows the choices made in TGB's experiments with TGN \citep{tgb-paper} for dynamic node affinity prediction:

\paragraph{Message Function.}
Our message function concatenates its inputs. For example, in the case of TGN: 
\begin{align*}
    \text{msg}_s(\mathbf{s}_i(t^-), \mathbf{s}_j(t^-), \phi(\Delta t), e_{ij}(t)) &= [\mathbf{s}_i(t^-) \circ \mathbf{s}_j(t^-) \circ \phi(\Delta t) \circ e_{ij}(t)]^T \\ 
    \text{msg}_d(\mathbf{s}_j(t^-), \mathbf{s}_i(t^-), \phi(\Delta t), e_{ij}(t)) &= [\mathbf{s}_j(t^-) \circ \mathbf{s}_i(t^-) \circ \phi(\Delta t) \circ e_{ij}(t)]^T 
\end{align*}
where $\circ$ is a concatenation operator. Since there are no node events in the TGB dataset, we chose to ignore formulating the message function for node events. 

\paragraph{Message Aggregator.}
We set the message aggregator to take the last message in a batch: 
\begin{align*}
    \bar{\mathbf{m}}_i(t) &= \text{agg}(\mathbf{m}_i(t_1), ..., \mathbf{m}_i(t_b)) \\
                          &= \mathbf{m}_i(t_b)
\end{align*}

\paragraph{Memory Updater.}
We set the memory updater to be a GRU:
\begin{align*}
    \mathbf{s}_i(t) &= \text{mem}(\bar{\mathbf{m}}_i(t), \mathbf{s}_i(t^-))) \\
                    &= \text{GRU}(\bar{\mathbf{m}}_i(t), \mathbf{s}_i(t^-)))
\end{align*}
where $\mathbf{s}_i \in \mathbb{R}^{d_{memory}}$.

\paragraph{Embedding}
We set the embedding module to be one layer of TransformerConv \citep{transformer-conv} with 2 heads and a dropout value of 0.1. For efficiency, we only use the last $x$ neighbours for temporal message passing. We describe the value of $x$ for each experiment in \Cref{hyperparam:tgb}. We set $\mathbf{z}_i(t)$ to have a dimension of $d_{embedding}$.

\paragraph{Decoder}
The decoder takes in the embedding for each node and outputs the node affinities for all other nodes. For our decoder, we chose an MLP with 2 layers + ReLU. Both layers have dimensionality $d_{decoder}$.

\paragraph{Time / Node Encoder}
We set $\phi_t(t) = \text{cos}(w_t \cdot t)$ and $\phi_n(i) = \text{cos}(w_n \cdot i)$ where $w_t \in \mathbb{R}^{d_{time}}$ and $w_n \in \mathbb{R}^{d_{node}}$. \\

\begin{table}[!hb]
  \caption{Hyperparameters for TGB experiments, for both TGN and TGNv2.}
  \label{hyperparam:tgb}
  \centering
  \resizebox{\textwidth}{!}{
      \begin{tabular}{c | c | c | c | c }
        \toprule
        & \texttt{tgbn-trade}  & \texttt{tgbn-genre} & \texttt{tgbn-reddit} & \texttt{tgbn-token} \\
        \midrule
        Learning Rate & 1e-3 & 1e-4 & 1e-4 & 1e-4 \\
        Batch Size & 200 & 200 & 200 & 200 \\ 
        Epochs & 750 & 50 & 50 & 50 \\
        $d$ & 784 & 784 & 784 & 1024 \\
        No. of temporal neighbours $x$ & 25 & 30 & 30 & 10\\
        \bottomrule
      \end{tabular}
    }
\end{table}

\paragraph{Hyperparameters}
For moving average over the ground-truth labels (Moving Average (L)), we set $k = 7$. For moving average over messages (Moving Average (M)), we set $k = 2048$ for \texttt{tgbn-trade}, \texttt{tgbn-genre}, \texttt{tgbn-reddit}, and $k = 512$ for \texttt{tgbn-token} due to memory issues. We use a constant learning rate schedule for all experiments, except for \texttt{tgbn-trade}, where we decay the learning rate by 0.5 every 250 epochs. We use the Adam optimiser \citep{adam} to train our models.  We set a global hidden dimension $d$, that is used in all places where we need to select a dimension, i.e. $d = d_{memory} = d_{embedding} = d_{decoder} = d_{time} = d_{node}$. The hyperparameters that we chose can be found in \Cref{hyperparam:tgb}.

\end{document}